\documentclass[twoside]{article} 

\usepackage{times}
\usepackage{graphicx} 
\usepackage{subfigure} 

\usepackage{natbib}

\usepackage{algorithm}
\usepackage{algorithmic}

\usepackage{amsfonts, amsmath, amssymb}

\usepackage{makecell}
\usepackage{booktabs}
\usepackage{multirow}
\usepackage{siunitx}

\usepackage{amsthm}

\usepackage{times}
\usepackage{epsfig}
\usepackage{graphicx}
\usepackage{amsmath}
\usepackage{amssymb}

\usepackage{caption}

\usepackage{makecell}
\usepackage{booktabs}
\usepackage{multirow}
\usepackage{siunitx}

\usepackage[T1]{fontenc}
\usepackage{fix-cm}

\newcolumntype{?}{!{\vrule width 1pt}}
\newcolumntype{C}[1]{>{\centering}m{#1}}

\newcolumntype{X}{@{\hskip\tabcolsep\vrule width 1.5pt\hskip\tabcolsep}}

\newcommand{\myfigurefivecol}[1]{
\begin{minipage}[b]{.08\textwidth}
\includegraphics[width=1.15\linewidth]{#1}
\end{minipage}
}

\newcommand{\myfigurebigfivecol}[1]{
\begin{minipage}[b]{.17\textwidth}
\includegraphics[width=1.055\linewidth]{#1}
\end{minipage}
}

\newcommand{\myfigurefivecolcaption}[2]{
\begin{minipage}[b]{.17\textwidth}
\includegraphics[width=1.055\linewidth]{#1}
\caption{{\small {#2}}}
\end{minipage}
}

\newcommand{\myfiguresmallfivecolcaption}[2]{
\begin{minipage}[b]{.08\textwidth}
\includegraphics[width=1.15\linewidth]{#1}
\caption{{\scriptsize {#2}}}
\end{minipage}
}

\newcommand{\myfiguresixcol}[1]{
\begin{minipage}[b]{.14\textwidth}
\includegraphics[width=1.085\linewidth]{#1}
\end{minipage}
}

\begingroup
    \makeatletter
    \@for\theoremstyle:=definition,remark,plain\do{%
        \expandafter\g@addto@macro\csname th@\theoremstyle\endcsname{%
            \addtolength\thm@preskip\parskip
            }%
        }
\endgroup

\newtheorem{thm}{Theorem}[section]

\newtheorem{mydef}[thm]{Definition}

\newcommand{\E}{\mathbb{E}}

\newcommand{\vx}{\boldsymbol{x}}
\newcommand{\vv}[1]{\boldsymbol{#1}}

\newcommand{\med}[1]{\textcolor{magenta}{#1}}

\newcommand{\prob}{\mathrm{Pr}}

\DeclareMathOperator*{\argmax}{arg\,max}

\usepackage{hyperref}

\usepackage[accepted]{aistats2017}
\begin{document}

%

%

\twocolumn[
\aistatstitle{Local Perturb-and-MAP for Structured Prediction}


\aistatsauthor{Gedas Bertasius \And Qiang Liu \And Lorenzo Torresani \And Jianbo Shi}
\aistatsaddress{ University of Pennsylvania \And Dartmouth College \And Dartmouth College \And University of Pennsylvania }

\vskip 0.3in
]

\begin{abstract} 

Conditional random fields (CRFs) provide a powerful tool for structured prediction, but cast significant challenges in both the learning and inference steps. Approximation techniques are widely used in both steps, which should be considered jointly to guarantee good performance (a.k.a. ``inferning"). 
Perturb-and-MAP models provide a promising alternative to CRFs, but require global combinatorial optimization and hence they are usable only on specific models. 
In this work, we present a new Local Perturb-and-MAP (locPMAP) framework that  replaces the global optimization with a local optimization by exploiting our observed connection between locPMAP and the pseudolikelihood of the original CRF model. 
%
We test our approach on three different vision tasks
%
and show that our method achieves consistently improved performance over other approximate inference techniques optimized to a pseudolikelihood objective.  Additionally, we demonstrate that we can integrate our method in the fully convolutional network framework to increases our model's complexity. Finally, our observed connection between locPMAP and the pseudolikelihood leads to a novel perspective for understanding and using pseudolikelihood. 

\end{abstract} 

\section{Introduction}
Probabilistic graphical models, such as Markov random fields, and conditional random fields \citep{Lafferty:2001:CRF:645530.655813} 
 provide a powerful framework for solving challenging learning problems that require structured output prediction~\citep{lauritzen1996graphical}. The use of graphical models consists of two main steps: \emph{learning}, which estimates the model parameters from the data, as well as \emph{inference}, which makes predictions based on the learned model.

Unfortunately, both maximum likelihood learning and probabilistic inference involve calculating the normalization constant (i.e. a partition function), which is intractable to compute in general. 
In practice, approximation methods, such as variational inference and MCMC, are widely used for inference and learning. There also exist other consistent learning methods such as the maximum pseudolikelihood (PL) estimator~\cite{Besag1975} 
which does not require calculating likelihood and is computationally tractable. 
%

 It is well known that there are strong interactions between learning and inference. As a result, the choice of the learning and inference algorithms should be considered jointly, an idea which is referred to as ``inferning".\footnote{see \url{http://inferning.cs.umass.edu}.} 
 Although it is relatively easy to identify ``inferning" pairs when using variational or MCMC approximations, it is unclear what the natural inference counterpart of pseudolikelihood (PL) is. For instance,
even if PL learning estimates a true model, a poor choice of a subsequent approximate inference algorithm may deteriorate the overall prediction accuracy.  

An alternative way to achieve ``inferning'' is to employ models that are computationally more tractable. One such framework is the \emph{Perturb-and-MAP} model \citep{conf/iccv/PapandreouY11, journals/jmlr/TarlowAZ12, NIPS2013_5066}
which involves injecting noise into the log-probability (the potential function),
and generating random samples to find the global maximum, i.e., the maximum \emph{a posteriori} (MAP) estimate of the perturbed probability.   
These models have a sound probabilistic interpretation, which can be exploited to make predictions with an uncertainty measure. Another benefit is that models from the  \emph{Perturb-and-MAP} class require combinatorial optimization, which is easier to solve than in the case of probabilistic inference models, which instead require marginalizing over variables or drawing MCMC samples. 


Unfortunately, 
despite being easier than marginalization, MAP estimation still requires a global optimization that is generally NP-hard. Thus, this prevents the use of Perturb-and-MAP for generic graphical models. 
In addition, the Perturb-and-MAP model can be viewed as a hidden variable model with deterministic constraints, and its training casts another challenging learning task that involves maximizing a non-convex likelihood function. This is often solved using  variants of EM, combined with approximation schemes \citep{DBLP:conf/aistats/GaneHJ14, journals/jmlr/TarlowAZ12}. 

In this work, we propose ``Local Perturb-and-MAP'' (locPMAP), a model that only requires finding a local maximum of the perturbed potential function, which is much easier than global optimization required for Global-MAP methods. 
%
Our locPMAP model has a close connection with the classical pseudolikelihood, 
in that pseudolikelihood can be interpreted as a type of partial information likelihood of our locPMAP model. 
This motivates us to decode pseudolikelihood training with a locPMAP inference procedure. 
We test our approach on three different vision tasks, where we show that locPMAP applied to models learned with PL yields consistently better inference results than those achieved using other approximate inference techniques.

In addition, we demonstrate that we can integrate our method in the fully convolutional network framework~\cite{long_shelhamer_fcn} to address the small complexity limitation of log-linear CRF models and improve performance on challenging structured prediction problems. 
Finally, our approach provides a novel view for pseudolikelihood and opens up opportunities for many useful extensions.


\section{Related Work}
The idea of Perturb-and-MAP was motivated by the classical ``Gumbel-Max trick" 
that connects the logistic function with discrete choice theory \citep{mcfadden1973conditional, YELLOTT1977109}. 
It was first applied to graphical model settings by \citet{conf/iccv/PapandreouY11} and \citet{DBLP:conf/icml/HazanJ12}. 
These studies gave rise to a rich line of research \citep[see e.g.,][]{NIPS2013_5066, DBLP:conf/aistats/GaneHJ14, journals/jmlr/TarlowAZ12}. 
We remark that these methods all require global optimization, in contrast with the local optimization in our method. 

The idea of ``inferning,'' which enforces the consistency between learning and inference, was probably first discussed by 
~\citet{journals/jmlr/Wainwright06}, who showed that when exact inference is intractable, it is better to learn a ``wrong" model to compensate the errors made by the subsequent approximate inference methods. 
Empirical analysis of influence  of learning and inference procedures can also be found in \citet{journals/ml/SuttonM09, gelfand2014bottom, mismatch}.
A line of  work has been developed to explicitly tune parameters in approximate inference procedures \citep{meshi2010learning, stoyanov2012minimum, domke2013learning}. 
In addition, 
\citet{Srivastava_fastinference} proposed an approximate inference method that interacts with learning in order to train Deep Boltzman Machines more efficiently. 
It is also relevant to mention \citep{poon2011sum}, which provides another class of models that enables efficient inference. 

\section{Background}
In subsection~\ref{crf_sec}, we introduce some background information on conditional random fields (CRFs). Additionally, in subsection~\ref{gumbel_sec}, we present some key ideas related to the Gumbel-Max trick and Perturb-and-MAP. We will use all of these ideas to introduce our method in Section~\ref{method_sec}.

\subsection{Structured Prediction with CRFs}
\label{crf_sec}

CRFs~\cite{Lafferty:2001:CRF:645530.655813} provide a framework to solve challenging structured prediction problems. 
Let $\vv x$ be an input (e.g., an image), and $\vv y \in \mathcal Y $ a set of structured labels (e.g., a semantic segmentation). A CRF assumes that the labels $\vv y$ are drawn from an exponential family distribution

\begin{equation} \label{eq:crf_prob}
p(\vv y | \vv x; ~w) = \frac{1}{Z(\theta)}\exp(\theta(\vv y, \vv x, w))
\end{equation}

where $\theta$ is a potential function and $w$ are model parameters that need to be estimated from data; the normalization constant $Z(\vv x, w) = \sum_{\vv y} \exp(\theta(\vv y, \vv x, w)$ is difficult to compute unless the corresponding graph is tree-structured.


Now let us assume that we are given a set of labeled training examples $\{\vx^i, \vv y^i \}$. A typical maximum likelihood estimator learns parameters $w$ by maximizing the log likelihood function:
\begin{align}
\hat w  = \argmax_{w} \sum_i \log p(\vv y^i | \vv x^i; w). \label{equ:MLEcrf061015}
\end{align}

With the estimated parameters $\hat w$  we can then make predictions for a new testing image $\vv x^*$: 
\begin{align}
\hat {\vv y} = \argmax_{\vv y}   p(\vv y | \vv x^*; \hat w) = \argmax_{\vv y}  \theta(\vv y , \vv x^*; \hat w).  \label{equ:MAPcrf061015}
\end{align}
However, both the learning and inference steps in Equations \eqref{equ:MLEcrf061015}-\eqref{equ:MAPcrf061015} are computationally intractable for general loopy graphs. Instead, a popular computationally-efficient alternative for MLE is the pseduolikelihood (PL) estimator~\citep{Besag1975}, defined as:

\begin{equation} \label{eq:pl}
\hat w  = \argmax_{w} \sum_i \sum_j  \log p(y^i_j | \vv x^i,  \vv y_{\neg j}^i; w), 
\end{equation}

where $\neg j$ refers to the neighborhood of the nodes in the graph that are connected to the node $j$. Based on this formulation, each conditional likelihood does not involve $Z$, and can be calculated efficiently. 
\citet{Besag1975} showed that PL is an asymptotically consistent estimator, meaning that $\hat w$  approaches the true parameter $w$ as the size of the dataset is increased. 

However, even if PL estimates the true parameters perfectly, the prediction step in \eqref{equ:MAPcrf061015} still requires approximation. Iterated Conditional Modes (ICM) is one of the simplest inference algorithm that returns a local maximum from the neighborhood of nodes. Other widely used inference techniques include loopy belief propagation (LBP), the mean field algorithm (MF), and Gibbs sampling.

The problem with using these approximate inference algorithms is that they may not work well together with the PL learning algorithm. 
This is because learning and inference are performed disjointly without considering how one may affect the other. 
Much of the difficulty comes from the fact that the definition of PL in \eqref{eq:pl} is not ``generative'', since the $\vv y^i$ depend on each other in a loopy fashion.  
\subsection{Perturb-and-MAP}
\label{gumbel_sec}
In contrast with the CRF defined in \eqref{eq:crf_prob}, 
the Perturb-and-MAP model \citep{DBLP:conf/icml/HazanJ12,conf/iccv/PapandreouY11} considers distributions of the form 
\begin{equation} \label{eq:gumbel_max}
\prob[\vv y \in \argmax\{\theta(\vv y)+\epsilon(\vv y)\}],~~~~ \epsilon \sim q  
\end{equation} 
where we dropped the dependency on $\vx$ to simplify the notation. That is, we first perturb the potential function $\theta(\vv y)$ with random noise $\epsilon(\vv y)$ from distribution $q$ and then draw the sample by finding the maximum point $\vv y$. 
Consider special perturbation noise $\epsilon(\vv y)$ drawn i.i.d. from the zero mean Gumbel distribution with cumulative distribution function $F(t) = \exp(- \exp( - (t + c))$, where $c$ is the Euler constant.
The Gumbel-Max trick \citep{YELLOTT1977109, mcfadden1973conditional} shows that the Perturb-and-MAP model is then equivalent to the distribution in \eqref{eq:crf_prob}, that is, 
\begin{equation} \label{eq:gumbel_max}
\prob[\vv y \in \argmax\{\theta(\vv y)+\epsilon(\vv y)\}] = \frac{\exp(\theta(\vv y))}{\sum_{\vv y} \exp(\theta(\vv y)))}. 
\end{equation}
 This connection provides a basic justification for Perturb-and-MAP models. It is also possible to use more general perturbations beyond the Gumbel perturbation, 
but then the training of the Perturb-and-MAP model becomes substantially more difficult, requiring EM-type non-convex optimization with Monte Carlo or other approximations 
\citep{journals/jmlr/TarlowAZ12, NIPS2013_5066, DBLP:conf/aistats/GaneHJ14}.

\section{Local Perturb-and-MAP Optimization}
\label{method_sec}

A major limitation of Perturb-and-MAP, even when using Gumbel noise, is that it requires global optimization over the perturbed potentials. 
We address this problem by replacing the global optimum with a local optimum. 
We start by defining the notion of local optimality. 


\renewcommand{\L}{\mathrm{Loc}}
\begin{mydef}
Let $\mathcal{B} = \{ \beta_k\}$ be a set of non-overlapping sets of variable indices such that $\beta_k \cap \beta_l = \emptyset $ for $\forall ~ k \neq l$.  Then we say that 
$$
\vv y \in \L \big[\theta(\vv y) ; ~ {\mathcal{B}} \big] 
$$
if $\vv y_{\beta} \in \argmax_{\vv y'_\beta} \big[ \theta(\vv y'_\beta, \vv y_{\neg \beta})\big]$ for $\forall \beta \in \mathcal{B}$, where $\neg \beta = [p] \setminus \beta$. This implies that 
$[\vv y_{\beta}, \vv y_{\neg \beta} ]$ is no worse than $[\vv y_{\beta}', \vv y_{\neg\beta}]$ for any $\vv y_{\beta}' \in Y_{\beta}$, $\beta \in \mathcal{B}$. In other words, $\vv y$ is a block-coordinate-wise maximum of $\theta(\vv y)$ on the set $\mathcal{B}$.  
\end{mydef}

We are now ready to establish our main result. We show that by exploiting random Gumbel perturbations over the potential functions, we can formulate a Local Perturb-and-MAP model, which yields a close connection with pseudolikelihood. 


\begin{thm} \label{thm:perturbicm061015}
Let us now perturb $\theta(\vv y)$ to get $\tilde\theta(\vv y) = \theta(\vv y) +\sum_{\beta\in \mathcal B } \epsilon(\vv y_{\beta})$, 
where each element $\epsilon(\vv y_{\beta})$ is drawn i.i.d. from a Gumbel distribution with CDF $F(t) =\exp(-\exp(-(t+c)))$, where $c$ is the Euler constant. 
Then we have, \text{ $\forall \vv y$},  
$$
\prob\bigg ( \vv y \in  \L\big[\tilde\theta(\vv y)  ; ~ {\mathcal{B}} \big]\bigg)= \prod_{\beta \in \mathcal{B}} p(  \vv y_{\beta} | \vv y_{\neg \beta} ; ~ \theta). 
$$
Note that the right hand side has a form of composite likelihood \citep{lindsay1988composite}, 
which reduces to the pseudolikelihood when taking $\mathcal B = \{k\colon k\in [n]\}$. 
\end{thm}
\begin{proof} 
Note that based on our definition of $\L()$ we can write $\L\big[\tilde\theta(\vv y)  ;  {\mathcal{B}} \big] = \cap_{\beta\in \mathcal{B}} A_{\beta}$, where 
$A_{\beta} = \{\vv y \colon  \vv y_{\beta} \in \argmax_{\vv y_{\beta}}\big[ \tilde \theta(\vv y_{\beta}, ~ \vv y_{\neg \beta}) \big] \big\}.$ Then, we can write:
\begin{align*}
\prob(\vv y \in  A_{\beta}  )
& = \prob (  \vv y_{\beta} \in \argmax_{ \vv y_{\beta}} \big [\theta(\vv y_{\beta}, ~ \vv y_{\neg \beta})  + \epsilon (\vv y_{\beta})  \big]  ) \\
& = \frac{\exp( \theta(\vv y_\beta, \vv y_{\neg\beta}))}{\sum_{\vv y'_{\beta}}   \exp( \theta(\vv y'_\beta, \vv y_{\neg\beta}) }  \\
& = p(\vv y_{\beta} | \vv y_{\neg \beta}; \theta), 
\end{align*}
where we use Equation~\ref{eq:gumbel_max} to derive these equalities. Note that Equation~\ref{eq:gumbel_max} results from the application of the Gumbel-Max trick. In the context of our problem, this equation holds because  $\epsilon(\vv y_{\beta})$ are drawn i.i.d. from a zero mean and a unit variance Gumbel distribution. Additionally, since  $\epsilon(\vv y_{\beta})$ are drawn independently from each other, the events $[\vv y \in A_\beta]$ are independent too. Therefore, we can write:
\begin{align*}
\prob(\vv y \in  \L\big[\tilde\theta(\vv y)  ; ~ {\mathcal{B}} \big])
& =  \prob( \vv y \in \cap_{\beta \in \mathcal{B}} A_{\beta}) \\ 
& = \prod_{\beta \in \mathcal{B}} p(\vv y_\beta | \vv y_{\neg\beta}; ~\theta). 
\end{align*}
\end{proof}
%
%
%
Our locPMAP model 
defines a procedure for generating \emph{random subsets} 
$\L\big[\tilde\theta(\vv y)  ; ~ {\mathcal{B}} \big]$ formed by the local maxima of the random function $\tilde \theta(\vv y)$. 
Theorem~\ref{thm:perturbicm061015} suggests that 
for any given (deterministic) configuration $\vv y$, 
the probability that $\vv y$ is an element of $\L\big[\tilde\theta(\vv y)  ; ~ {\mathcal{B}} \big]$ equals 
the composite likelihood 
$\ell_{\mathcal{B}}(\vv y;~\theta)~  \colon\!\!\!\!= \prod_{\beta \in \mathcal{B}} p(  \vv y_{\beta} | \vv y_{\neg \beta} ; ~ \theta)$. 
Here the point $\vv y$ is deterministic, 
while the set $\L\big[\tilde\theta(\vv y)  ; ~ {\mathcal{B}} \big]$ is random, 
similar to the case of confidence intervals in statistics. 

We should point out that $\ell_{\mathcal{B}}(\vv y; ~\theta)$ is not a properly normalized distribution over $\vv y\in \mathcal Y$, 
because there may be multiple local maxima in each random set $\L\big[\tilde\theta(\vv y)  ; ~ {\mathcal{B}} \big]$. 
In fact, it is easy to see that the expected number $Z_{\mathcal  B}$ of local maxima of $\tilde\theta(\vv y)$ is 
\begin{align}\label{equ:zb}
Z_{\mathcal B} \overset{def}{=} \E (\big|\L\big[\tilde\theta(\vv y)  ; ~ {\mathcal{B}} \big] \big|)  
= \sum_{\vv y\in \mathcal Y}  \prod_{\beta \in \mathcal{B}} p(\vv y_\beta | \vv y_{\neg\beta}; ~\theta). 
\end{align}
This can be used to define a normalized probability over $\mathcal Y$ via $\ell_{\mathcal{B}}(\vv y; ~\theta)/Z_{\mathcal B}$, which, however, is computationally intractable due to the difficulty for computing $Z_{\mathcal B}$. 
 
It is interesting to draw a comparison with the global Perturb-and-MAP model  
when $\mathcal B$  includes only the global set of all the elements of $\vv y$, 
in which case we can show that $Z_\mathcal{B} = \E (\big|\L\big[\tilde\theta(\vv y)  ; ~ {\mathcal{B}} \big] \big|) = 1$ in \eqref{equ:zb}. 
Because there always exists at least one optimum point, we must always have $\big|\L\big[\tilde\theta(\vv y)  ; ~ {\mathcal{B}} \big] \big|\geq 1$.  
Therefore, we have $\big|\L\big[\tilde\theta(\vv y)  ; ~ {\mathcal{B}} \big] \big| =1$ with probability $1$ in this case, that is, there only exists one unique global optimum point. 

%
%
%
%
%


\paragraph{locPMAP and Pseudolikelihood}
In practice, we can not exactly enumerate, nor observe the whole set 
$\mathcal L  = \L\big[\tilde\theta(\vv y)  ; ~ {\mathcal{B}} \big] $. 
We instead observe a single point $\vv y$ which we can assume to belong to $\mathcal L$. 
Without further assumption on how $\vv y$ is selected from $\mathcal L$, 
the only information available through observing a point $\vv y$ is $\prob(\vv y \in \mathcal L)$. 
As a result, given a set of i.i.d. observation $\{\vv y_i\}$, it is natural to maximize their overall observation likelihood: 
\begin{align*}
\hat w
&  = \argmax_{w} \sum_{i} \log p(\vv y^i \in  \L\big[\tilde\theta(\vv y;~ w)  ; ~ {\mathcal{B}} \big] )  \\
& = \argmax_{w} \sum_i \sum_{\beta \in \mathcal B} 
\log p(\vv y_\beta^i |\vv y_{\neg \beta}^i; ~w ). 
\end{align*}
This formulation interprets maximum composite likelihood (CL) \citep{lindsay1988composite} as a type of partial information likelihood for locPMAP. 
Here locPMAP 
is not a complete generative model in terms of the observation points $\vv y$. 
Although it is possible to complete the model by defining a specific mechanism to map from the local maxima set $\mathcal L$
to point $\vv y$, the corresponding full likelihood would become 
more challenging to compute and analyze. 
We can see that the essence of CL and PL is to trade off information for computational tractability, 
which is also reflected in the original definition \eqref{eq:pl} of PL where only the conditional information $p(y_j |\vv y_{\neg j}; ~ w)$ is taken into account. 
This also makes CL/PL more robust than the full MLE 
when the full model is misspecified, but the partial information used by CL/PL 
is correct \citep{xu2011robustness}. 
Our perspective 
motivates us to decode CL/PL, interpreted as training locPMAP, 
by mimicking the locPMAP procedure: 
we first generate a randomly perturbed potential function $\tilde\theta(\vv y) $ using i.i.d. Gumbel noise, 
and then find a local maximum with an arbitrary greedy optimization procedure (such as ICM with uniform random initialization). 
Although we do need to specify a particular greedy optimization method to select $\vv y$ from $\mathcal L$ for the purpose of inference, 
this seems to be a minor approximation, and may not influence the result significantly unless the selected greedy optimization 
is strongly biased in a certain way. 
 
We outline these two steps in Algorithm~\ref{algo}. 
In this case, we assume a simple case 
of likelihood and  $\mathcal B$ consists of  the set of single variables. 
For simplicity, from now on we drop the dependency on $\mathcal B$ and simply use the notation $\L [\tilde \theta(\vv y, \vx;~w )]$,
where the dependency on the input $\vv x$ is added explicitly.

Because the result of Algorithm~\ref{alg:perturbicml051015} is not deterministic, we can repeatedly run it for several iterations (by drawing multiple samples from our locPMAP model), and then take the mode of the returned samples. 
This also allows us to construct probabilistic outputs with error bars indicating the variance of the structured prediction. Note that in order to obtain analogous probabilistic results with other models, it would be necessary to perform expensive MCMC inference for the case of CRFs, or global combinatorial optimization under the typical Perturb-and-MAP models. 


\begin{algorithm}[t] 
\caption{Local Perturb-and-MAP (locPMAP)}  \label{alg:perturbicml051015}
\begin{algorithmic}
\STATE 1. Let $\tilde \theta(\vv y, \vx ;~w ) =  \theta(\vv y, \vx ;~w )  + \epsilon(\vv y)$, where $\epsilon(\vv y)$ are drawn i.i.d. from Gumbel $\sim G(0, 1)$. 
\STATE 2. Run an greedy optimization method (such as ICM) on the perturbed potentials $\tilde \theta (\vv y, \vx ;~w)$ to get $y_j  = \argmax_{y_j} \tilde \theta(y_j, \vv y_{\neg j}, \vx ;~w), ~~ \forall j $. 
\label{algo}
\end{algorithmic}
\end{algorithm}

\section{Learning Deep Unary Features with FCNs and Pseudolikelihood Loss}


\textbf{Background.} Under log-linear CRF models, we typically assume that the potential function is a weighted combination of features:



$$
\theta^{U}_i = \exp{\sum_{j} w^{U}_j f^{U}_j (x_i)}~~~~~~ \theta^{P}_{ij}= \exp{ \sum_{k} w^{P}_k f_k^{P} (x_i,x_j)}
$$


where $\theta^{U}_i$,  $w^{U}$ and $f^{U} (x_i)$  denote unary potentials for node $i$, learnable unary feature parameters, and unary features, respectively. Similarly, $\theta^{P}_{ij}$,  $w^{P}$ and $f^{P} (x_i,x_j)$ are pairwise potentials between nodes $i$ and $j$, learnable pairwise feature parameters, and pairwise features, respectively.

However, there are several important limitations related to log-linear CRF models. Such models have only a linear number of learnable parameters, which significantly limits the complexity of the model that can be learned from the data. One way to address this limitation is to construct highly non-linear and complex features that would work well even with a linear classifier. However, hand-engineering complex features is a challenging and time-consuming task requiring lots of domain expertise. To address both of these limitations, we train a deep network that optimizes a pseudolikelihood criterion and automatically learns complex unary features for our model. 




Recently, deep learning methods have been extremely successful in learning effective hierarchical features that achieve state-of-the-art results on a variety of vision tasks, including boundary detection, image classification, and semantic segmentation~\cite{gberta_2015_ICCV,NIPS2012_4824,DBLP:journals/corr/DonahueJVHZTD13,DBLP:journals/corr/ToshevS13,deepface}. A particularly useful model for structured prediction on  images is the Fully Convolutional Network (FCN)~\cite{long_shelhamer_fcn} used in combination with CRF models. These models combine the powerful methodology of deep learning for hierarchical feature learning with the effectiveness of CRFs for modeling structured pixel output, such as the class labels of neighboring pixels in semantic segmentation. 

While in early approaches the FCN and the CRF were learned separately~\cite{chen14semantic}, more recently there has been successful work that has integrated CRF learning into the FCN framework~\cite{crfasrnn_iccv2015}. Additionally,  learning the parameters of the CRF in the neural network model has been addressed in~\cite{NIPS2009_3869,Do_AISTATS_2010,DBLP:journals/corr/KirillovSFZ0TR15}. 


In our approach, we train FCN and CRF jointly by optimizing the entire FCN via backpropagation with respect to the pseudolikelihood loss as explained below. 
We then use the locPMAP procedure shown in Algorithm~\ref{alg:perturbicml051015} to decode the PL result, as justified by our intuition discussed earlier. 






\textbf{Optimizing the Pseudolikelihood Loss.}  Let our input be an image of size $h \times w \times c$, where $h,w$ refer to the height and width of the image and  $c$ is the number of input channels ($c=3$ for color RGB images, $c=1$ for grayscale images). Then assume that our goal is to assign one of $K$ possible labels to each pixel $(i,j)$. The label typically denotes the class of the object located at pixel $(i,j)$ or the foreground/background assignment. Now let us write our conditional pseudolikelihood probability as:



\begin{equation}\label{eq:sm}
p(y_{(i,j)} = l ~|~\vv x, \vv y_{\neg {(i,j)}}) =  \frac{\exp( \theta_{i,j,l})}{\sum_{k=1}^K   \exp( \theta_{i,j,k})}
\end{equation}

where $\theta_{i,j,l}$ refers to the potential function values for label $l$ at pixel $(i,j)$, and where $\neg{(i,j)}$ indicates all the nodes connected to node $(i,j)$. More specifically, $\theta_{i,j,l}$ denotes the product of a unary potential at a node $(i,j)$ and all the pairwise potentials that are connected to the node $(i,j)$. The subscript $l \in \{1, \hdots, K\}$ in the probability notation denotes that this is a potential associated with the class label $l$. Then, to obtain a proper probability distribution we can normalize this potential value as shown in Equation~\ref{eq:sm}. Finally, we can write the loss of our FCN as:

\begin{equation}\label{eq:loss}
L_{i,j,l}= - \log{ p(y_{(i,j)} = l ~|~\vv x, \vv y_{\neg {(i,j)}};~\theta_{i,j,l})}
\end{equation}

The gradient of this loss can then be computed as:

\begin{equation}\label{eq:loss}
\frac{\partial{L_{i,j,l}}}{\partial \theta_{i,j,l} }= p(y_{(i,j)} = l ~|~\vv x, \vv y_{\neg {(i,j)}}) - 1\{y_{i,j}=l\}
\end{equation}

where the last term in the equation is simply an indicator function denoting whether ground truth label $y_{i,j}$ is equal to the predicted label $l$. This gradient is computed for every node $(i,j)$ and is then backpropagated to the previous layers of the FCN. We provide more details about our choice of deep architecture and the other learning details in the experimental section.


\section{Experimental Results}

In this section, we evaluate the results of our Local Perturb-and-MAP (locPMAP) method against  other inference techniques 
such as loopy belief propagation (LBP) and mean field (MF) 
on three different datasets. In all our experiments, we use the following setup. First, we learn the parameters of a CRF-based model using the PL learning criterion. We note that the PL learning is done once, and the same learned parameters are then used for both our method and the other baseline inference techniques. 
This is done to demonstrate that our locPMAP procedure acts as a better inference procedure than existing approximation inference techniques. 

Since our method relies on ICM to make predictions, we compare our approach with traditional ICM. We also compare against an iterative version of ICM (ICM-iter) which is executed for the same number of iterations as our method, in order to give both methods the same ``computational budget." In this iterative version of ICM, at each iteration we randomly perturb the potentials by setting a small fraction (e.g., $0.1$) of them to zero (the technique is known as dropout in the deep learning literature~\cite{JMLR:v15:srivastava14a}). In all experiments, for our method and ICM-iter, we only perturb the unary potentials. Additionally, we use a grid-based graph model, with each node connected to its $4$ neighbors, as this is standard for computer vision problems. The details of the pairwise potentials are discussed separately below for each task. We also tested inference of the learned model using loopy belief propagation (LBP), mean field (MF), simulated annealing (SA) and Gibbs sampling (MCMC). For each of the three tasks we show that our Local Perturb-and-MAP method consistently outperforms other inference techniques, 
thus demonstrating that locPMAP forms a better practical inference procedure for models learned from PL optimization. 
We now present each of our experiments in more detail.



\setlength{\tabcolsep}{4pt}
   
\captionsetup[table]{aboveskip=0pt}

   \begin{table}[t]
   \small
    \begin{center}
    \begin{tabular}{  c | c | c | c | c | c | c |}
    \cline{2-7}
    & \multicolumn{2}{ c |}{Background} & \multicolumn{2}{ c |}{Foreground}  & \multicolumn{2}{ c |}{Mean}\\
    \cline{2-7}
    & \multicolumn{1}{ c |}{Raw} & \multicolumn{1}{ c |}{Deep} &  \multicolumn{1}{ c |}{Raw} &  \multicolumn{1}{ c |}{Deep}  &  \multicolumn{1}{ c |}{Raw}  &  \multicolumn{1}{ c |}{Deep}  \\ \cline{1-7}
      \multicolumn{1}{| c |}{LBP} & 0.846 & \bf 0.854 & 0.029 & 0.112 & 0.438 & 0.446 \\ \hline
       \multicolumn{1}{| c |}{MF} & \bf 0.861 & 0.837 & 0.059 & 0.209 & 0.460 & 0.523 \\ \hline
           \multicolumn{1}{| c |}{ICM-iter} & 0.722  & 0.797 & 0.312 & 0.379 & 0.517 & 0.588\\ \hline
         \multicolumn{1}{| c |}{SA} & 0.806 & 0.837 & 0.234 & 0.202 & 0.520 & 0.520 \\ \hline
          \multicolumn{1}{| c |}{ICM} & 0.730  & 0.807 & 0.319 & 0.389 & 0.525 & 0.598\\ \hline
           \multicolumn{1}{| c |}{Gibbs} & 0.840 & 0.840 & 0.233 & 0.197 & 0.537 & 0.518 \\ \hline
            \multicolumn{1}{| c |}{LocPMAP} & 0.753 & 0.826 &\bf 0.337 & \bf 0.404 & \bf 0.545 & \bf 0.615\\ \hline
    \end{tabular}
    \end{center}
    \caption{Results of handwritten digit denoising on the MNIST dataset. Performance is measured according to the the Intersection over Union (IoU) for both the foreground and the background mask. We compare the results when using raw corrupted pixel intensities as unary features (Raw) versus deep unary features learned via an FCN (Deep). Our locPMAP method outperforms the other baseline inference techniques. Additionally, we observe that using deep features tends to  improve the overall accuracy.\vspace{-0.5cm}}
    \label{mnist_table}
   \end{table}

\captionsetup[figure]{labelformat=empty}
\captionsetup[figure]{skip=1pt}

\begin{figure}[tb]
\centering

\myfigurefivecol{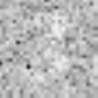}
\myfigurefivecol{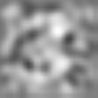}
\myfigurefivecol{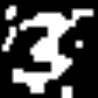}
\myfigurefivecol{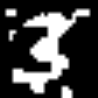}
\myfigurefivecol{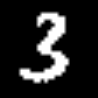}

\myfigurefivecol{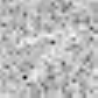}
\myfigurefivecol{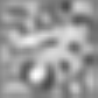}
\myfigurefivecol{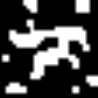}
\myfigurefivecol{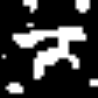}
\myfigurefivecol{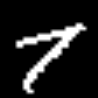}

\myfigurefivecol{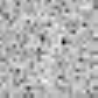}
\myfigurefivecol{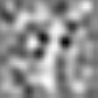}
\myfigurefivecol{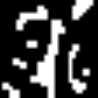}
\myfigurefivecol{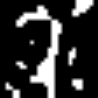}
\myfigurefivecol{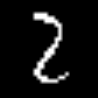}

\myfigurefivecol{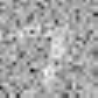}
\myfigurefivecol{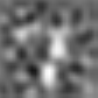}
\myfigurefivecol{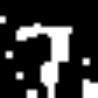}
\myfigurefivecol{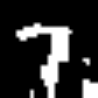}
\myfigurefivecol{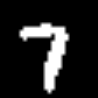}

\myfiguresmallfivecolcaption{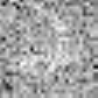}{Input}
\myfiguresmallfivecolcaption{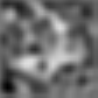}{FCN}
\myfiguresmallfivecolcaption{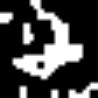}{ICM}
\myfiguresmallfivecolcaption{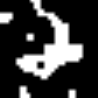}{LocPMAP}
\myfiguresmallfivecolcaption{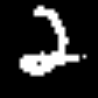}{True Mask}

\captionsetup{labelformat=default}
\setcounter{figure}{0}
    \caption{Visualizations of handwritten digit denoising. Images in the first column represent corrupted digit inputs. The second column shows the unary features that were learned using fully convolutional networks and a pseudolikelihood loss. Images in the third column correspond to ICM predictions, whereas the fourth column depicts our Local Perturb-and-MAP results. In the last column we show the corresponding original ground truth black and white images. Note that compared to the ICM predictions, our method makes fewer false positive predictions. Additionally, we observe that deep features exhibit significantly less noise than the original input, which improves the overall accuracy (see quantitative results in Table \ref{mnist_table}).\vspace{-0.5cm}}
    \label{fig_mnist}
\end{figure}

\subsection{Denoising Handwritten Digits}

For our first evaluation we use the MNIST dataset~\cite{lecun-mnisthandwrittendigit-2010} which contains black and white images of handwritten digits. We 
corrupt each $28 \times 28$ image using Gumbel noise with $0.25$ signal-to-noise ratio. This produces corrupted grayscale images, which are used as input to our system. The objective is to recover the original black/white (background/foreground) value of each pixel.

In our experiments, we used $5000$ images for training and $5000$ images for testing. We performed two types of experiments on this task. First, we evaluated all methods using the corrupted pixel intensity values as unary features. For pairwise features between nodes $i$ and $j$, we used the corrupted intensity values at pixels $i$ and $j$.

For the second experiment, we trained a fully convolutional network (FCN) to learn the unary features. To optimize the pseudolikelihood criterion, we used pairwise potential parameters that were learned using the corrupted potentials. We kept the pairwise potential parameters fixed and performed gradient backpropagation only through the unary feature parameters.

To train the FCN, we used an architecture composed of $5$ convolutional layers with kernel size of $3 \times 3$ for the first four layers and kernel size of $1 \times 1$ for the last layer. The output plane dimensions for the convolutional layers were $64,126,256,512$ and $2$ respectively. As hyperparameters, we used a learning rate of $10^{-6}$, a momentum of $0.9$, a batch size of $100$, and RELU non-linear functions in between the convolutional layers. To avoid the reduction in resolution inside the deep layers, we did not use any pooling layers. We trained our FCN to minimize the pseudolikelihood loss for $\approx 3000$ iterations. For all of our deep learning experiments we used the Caffe library~\cite{jia2014caffe}. To run our locPMAP method, we used $50$ iterations.

In Table~\ref{mnist_table}, we present quantitative results of our method and several  baseline methods. The performance of each method is evaluated in terms of the Intersection over Union (IoU) metric for background and foreground classes. Additionally, we separately evaluate each baseline inference method using corrupted pixel values as unary features (Raw) and also using deep features (Deep). The results demonstrate that our locPMAP method outperforms the other inference baselines. 
 Additionally, we note that learning unary features via FCNs substantially improves the accuracy for most methods.  We also present qualitative results in Figure~\ref{fig_mnist}.


We also tested other inference methods such as tree-reweighted belief propagation, and graph cuts, 
but found that these methods performed very poorly with PL learning. 
Due to the limited space available, we omit these results from our paper.


\subsection{Caltech Silhouette Reconstruction}

   \begin{table}
   \small
    \begin{center}
    \begin{tabular}{  c | c | c | c | c | c | c |}
    \cline{2-7}
    & \multicolumn{2}{ c |}{Background} & \multicolumn{2}{ c |}{Foreground}  & \multicolumn{2}{ c |}{Mean}\\
    \cline{2-7}
    & \multicolumn{1}{ c |}{Raw} & \multicolumn{1}{ c |}{Deep} &  \multicolumn{1}{ c |}{Raw} &  \multicolumn{1}{ c |}{Deep}  &  \multicolumn{1}{ c |}{Raw}  &  \multicolumn{1}{ c |}{Deep}  \\ \cline{1-7}
      \multicolumn{1}{| c |}{LBP} & 0.539 & 0.539 & 0.080 & 0.154 & 0.310 & 0.347 \\ \hline
        \multicolumn{1}{| c |}{MF} & 0.686  & 0.695 & 0.598 & 0.674 & 0.642 & 0.684\\ \hline
        \multicolumn{1}{| c |}{ICM-iter} & 0.659  & 0.731 & 0.637 & 0.727 & 0.648 & 0.729\\ \hline
        \multicolumn{1}{| c |}{SA} & 0.688  & 0.692 & 0.623 & 0.696 & 0.655 & 0.694\\ \hline
          \multicolumn{1}{| c |}{Gibbs} & \bf 0.692  & 0.708 & 0.624 & 0.696 & 0.658 & 0.702\\ \hline
          \multicolumn{1}{| c |}{ICM} & 0.672  & 0.746 & 0.661 & 0.733 & 0.667 & 0.739\\ \hline
            \multicolumn{1}{| c |}{ LocPMAP} & 0.681 & \bf 0.754 &\bf 0.666 & \bf 0.735 & \bf 0.673 & \bf 0.745\\ \hline
    \end{tabular}
    \end{center}
    \caption{Results on the Caltech Silhouette Reconstruction task. We evaluate the results using the Intersection over Union (IoU) metric. We test each method using raw unary features versus deeply learned FCN features. Our method achieves better accuracy than the other baseline inference methods.\vspace{-0.4cm}}
    \label{caltech_table}
   \end{table}

For our second task, we choose a more diverse dataset consisting of noise-corrupted silhouettes generated from the ground truth foreground/background segmentations of images from Caltech-$101$~\cite{Fei-Fei:2007:LGV:1235884.1235969}, which spans 101 object classes. Each silhouette is a $28 \times 28$ image generated by adding Gaussian noise with $0.5$ signal-to-noise ratio to each pixel of the ground-truth foreground/background segmentation. The goal is to reconstruct the original ground truth foreground/background segmentation from the corrupted silhouette. Due to the large number of object classes in the dataset, the variability of the silhouette shape is much larger compared to the case of the digit denoising task.

We use the exact same experimental setup as in the earlier experiment for handwritten digit denoising. We present quantitative results in Table~\ref{caltech_table}. 
Again, the results indicate that locPMAP outperforms the other inference baselines for the model learned from PL optimization. 


\captionsetup[figure]{labelformat=empty}
\captionsetup[figure]{skip=5pt}

\begin{figure*}
\centering

\myfigurebigfivecol{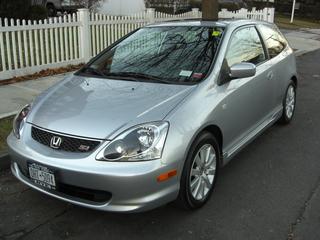}
\myfigurebigfivecol{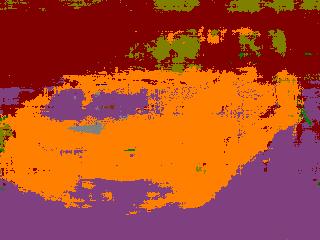}
\myfigurebigfivecol{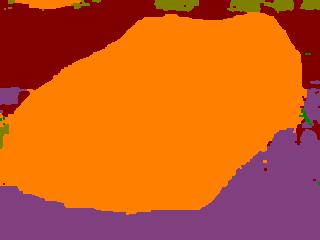}
\myfigurebigfivecol{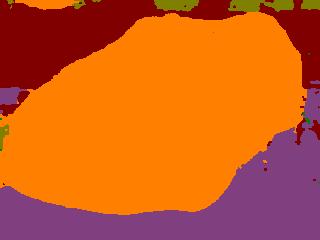}
\myfigurebigfivecol{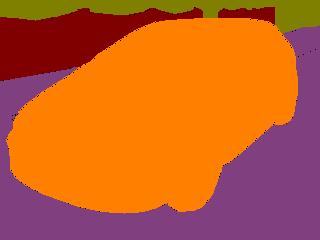}

\myfigurefivecolcaption{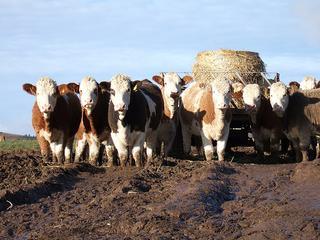}{Input Image}
\myfigurefivecolcaption{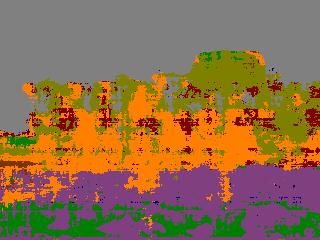}{Boosted Classifier}
\myfigurefivecolcaption{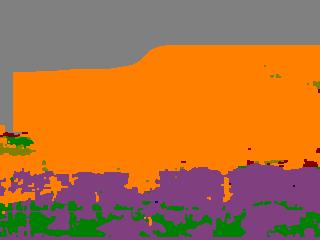}{ICM}
\myfigurefivecolcaption{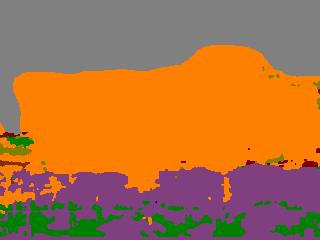}{LocPMAP}
\myfigurefivecolcaption{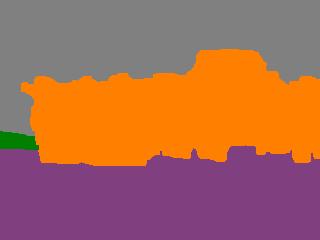}{Ground Truth}




\captionsetup{labelformat=default}
\setcounter{figure}{1}
    \caption{A figure illustrating qualitative results for the scene labeling task. In the first column, we show the original RGB input images. The second column represents boosted classifier predictions~\cite{Gould+al:ICCV09} while in the third column, we illustrate ICM predictions. In the fourth column, we present the results of our LocPMAP method. Notice that, compared to the other methods, our predictions are spatially smoother and crispier around the object boundaries.\vspace{-0.2cm}}
    \label{fig_stanford}
\end{figure*}

\subsection{Scene Labeling}

 \setlength{\tabcolsep}{2pt}

   \begin{table*}
   \small
    \begin{center}
    \begin{tabular}{ c | c | c | c | c | c | c | c | c | c | c | c | c | c | c | c | c | c | c |}
    \cline{2-19}
    & \multicolumn{2}{ c |}{Sky} & \multicolumn{2}{ c |}{Tree}  & \multicolumn{2}{ c |}{Road}  & \multicolumn{2}{ c |}{Grass} & \multicolumn{2}{ c |}{Water}  & \multicolumn{2}{ c |}{Building} & \multicolumn{2}{ c |}{Mountain} & \multicolumn{2}{ c |}{Object}  & \multicolumn{2}{ c |}{Mean}\\
    \cline{2-19}
    & \multicolumn{1}{ c |}{Raw} & \multicolumn{1}{ c |}{Deep} &  \multicolumn{1}{ c |}{Raw} &  \multicolumn{1}{ c |}{Deep}  &  \multicolumn{1}{ c |}{Raw}  &  \multicolumn{1}{ c |}{Deep}   & \multicolumn{1}{ c |}{Raw} & \multicolumn{1}{ c |}{Deep} &  \multicolumn{1}{ c |}{Raw} &  \multicolumn{1}{ c |}{Deep}  &  \multicolumn{1}{ c |}{Raw}  &  \multicolumn{1}{ c |}{Deep}  & \multicolumn{1}{ c |}{Raw} & \multicolumn{1}{ c |}{Deep} &  \multicolumn{1}{ c |}{Raw} &  \multicolumn{1}{ c |}{Deep}  &  \multicolumn{1}{ c |}{Raw}  &  \multicolumn{1}{ c |}{Deep}\\ \cline{1-19}
      \multicolumn{1}{| c |}{ICM} & 0.807 & 0.843 & 0.648 & 0.628 & 0.845 & 0.853 & 0.790 & 0.729 & 0.770 & 0.759 & \bf 0.686 & 0.724 & 0.343 & 0.175 & 0.633 & 0.668 & 0.690 & 0.672 \\ \hline
      \multicolumn{1}{| c |}{ICM-iter} & 0.837 & \bf 0.846 & 0.649 & 0.631 & 0.845 & 0.855 & 0.790 & 0.736 & 0.776 & 0.767 & \bf 0.686 & 0.726 & 0.421 & 0.233 & 0.646 & 0.684 & 0.706 & 0.685 \\ \hline
      \multicolumn{1}{| c |}{LBP} & \bf 0.861 & 0.840 & 0.640 & 0.629 & 0.832 & 0.842 & 0.783 & 0.690 & 0.771 & 0.707 & 0.675 & 0.708 & \bf 0.522 & 0.157 & 0.646 & 0.552 & 0.716 & 0.638 \\ \hline
       \multicolumn{1}{| c |}{Gibbs} & 0.854 & 0.840 & \bf 0.659 & 0.632 & \bf 0.848 & 0.856 & 0.792 & 0.738 & \bf 0.802 & 0.766 & 0.705 & 0.728 & 0.439 & 0.215 & 0.646 & 0.690 & 0.718 & 0.683 \\ \hline
      \multicolumn{1}{| c |}{ LocPMAP} & 0.854 & \bf 0.846 & 0.650 & \bf 0.636 & 0.844 & \bf 0.859 & \bf 0.794 & \bf 0.742 & 0.784 & \bf 0.769 & \bf 0.686 & \bf 0.729 & 0.512 & \bf 0.259 & \bf 0.652 & \bf 0.701 & \bf 0.722 & \bf 0.693 \\ \hline
    \end{tabular}
    \end{center}
    \caption{Quantitative results for the scene labeling task. The performance is evaluated using an IoU metric for each class. For raw features we use the output predictions from~\cite{Gould+al:ICCV09}. In this case, we observe that raw unary features yield higher accuracy in comparison to deep unary features, probably due to the small dataset size. However, we observe again that our LocPMAP is overall the best approach across all inference methods considered here.\vspace{-0.4cm}}
    \label{stanford_table}
   \end{table*}

As our last task, we consider the problem of semantic scene segmentation. For this task, we use the Stanford background dataset~\cite{Gould+al:ICCV09}, which has per-pixel annotations for a total of $715$ scene color images of size $240 \times 320$. Our goal is to assign every pixel to one of $8$ possible classes: sky, tree, road, grass, water, building, mountain, and foreground. In this case the input to the system is the RGB photo and the desired output is the semantic segmentation. We randomly split the dataset into a training set of $600$ images and a test set of $115$ images.

Once again we perform two experiments for this task. First, we use the boosted unary potentials provided by~\citet{Gould+al:ICCV09} as the unary features in our CRF model. Next, to construct pairwise potentials we extract HFL boundaries~\cite{gberta_2015_ICCV} from the images. We then compute the gradient on the boundaries, and use it as pairwise features for every pair of adjacent pixels. Then, just as earlier, we learn the CRF parameters by optimizing the pseudolikelihood objective, and finally perform the inference using the learned parameters. 

For the second experiment, our goal is to learn deep features from the data instead of using the boosted features provided by~\citet{Gould+al:ICCV09}. To do this we use a fully convolutional network architecture based on DeepLab~\cite{chen14semantic}. This architecture contains $19$ convolutional layers. To train the FCN, we use the same learning hyperparameters and setup as in the previous two experiments. As before, we fix the pairwise potential parameters, and only learn the parameters associated with the unary terms. After the unary learning is done, we learn new pairwise parameters given the learned unary features.

In Table~\ref{stanford_table}, we present our results for the scene labeling task. Once again, we show that our Local Perturb-and-MAP method outperforms other inference methods in both scenarios: using boosted unary features~\cite{Gould+al:ICCV09} and also using our learned deep features. 

Interestingly, we note that for this task, the accuracy we achieve using deep features is lower than the accuracy obtained using boosted features. We hypothesize that this happens because the Stanford Background dataset is relatively small ($600$ training images) and thus it does not enable effective training of the large-capacity FCN.

\begin{figure*}
\centering

\myfiguresixcol{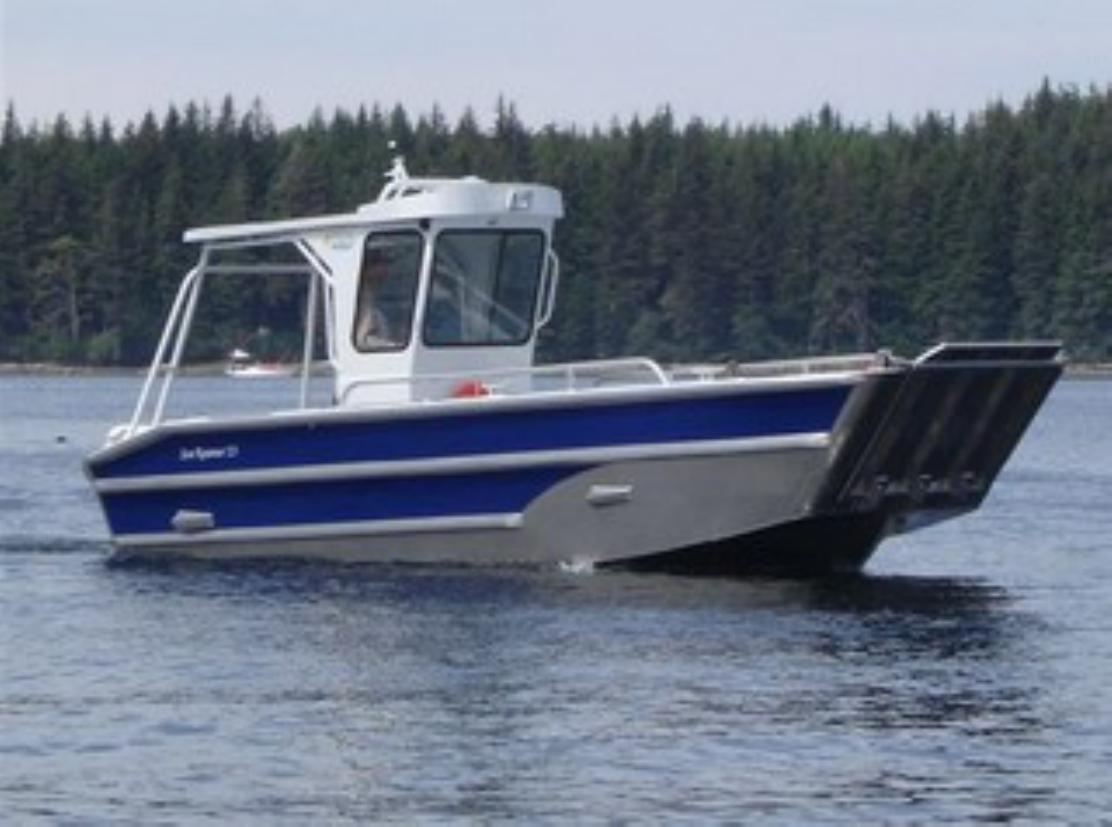}
\myfiguresixcol{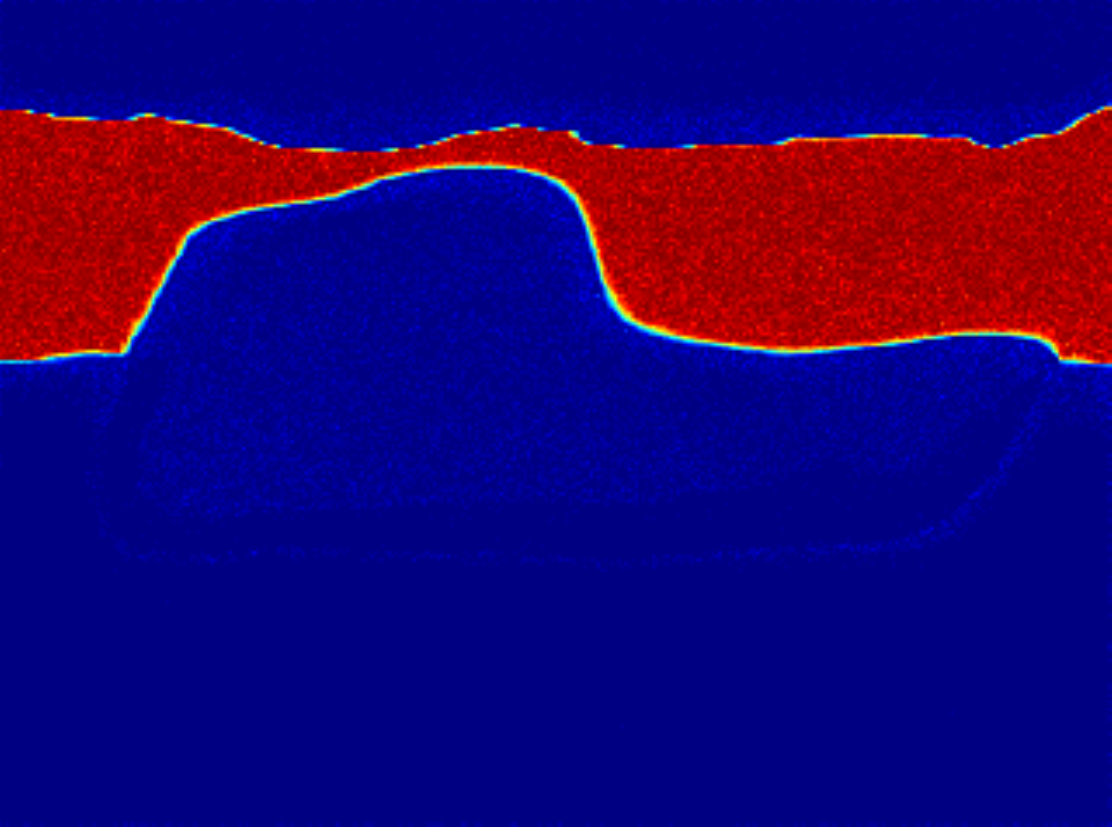}
\myfiguresixcol{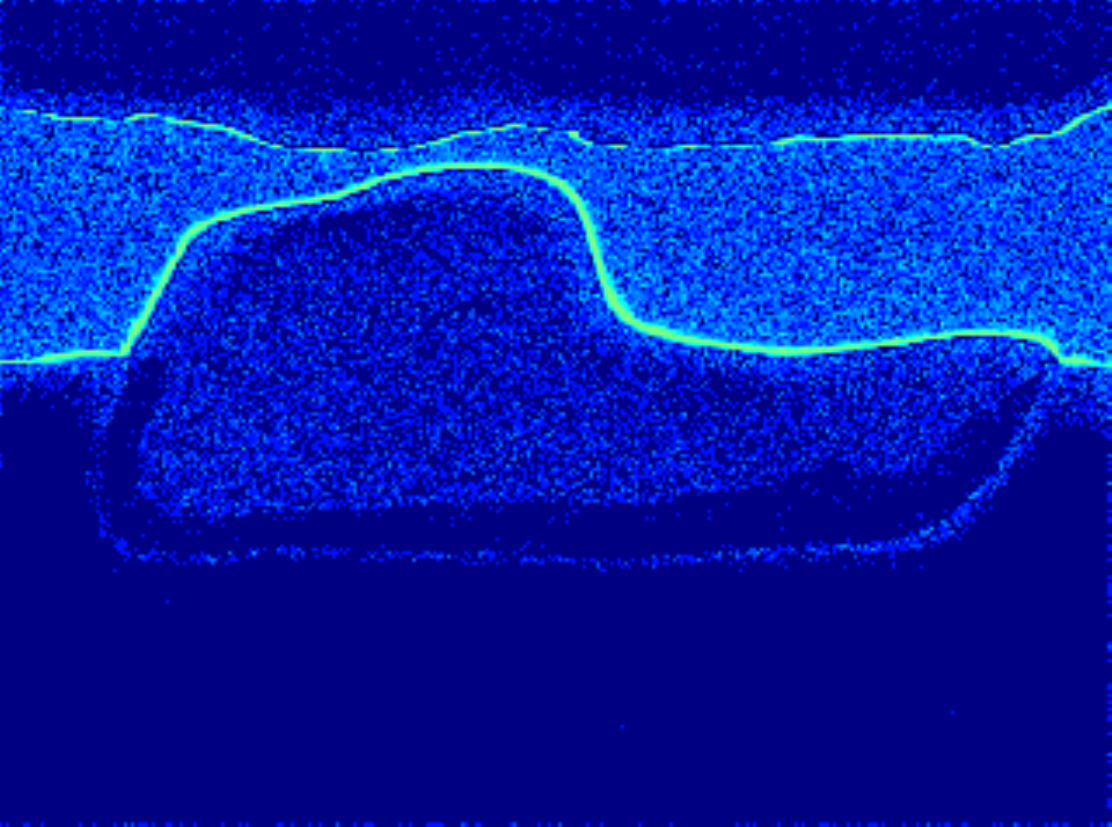}
\myfiguresixcol{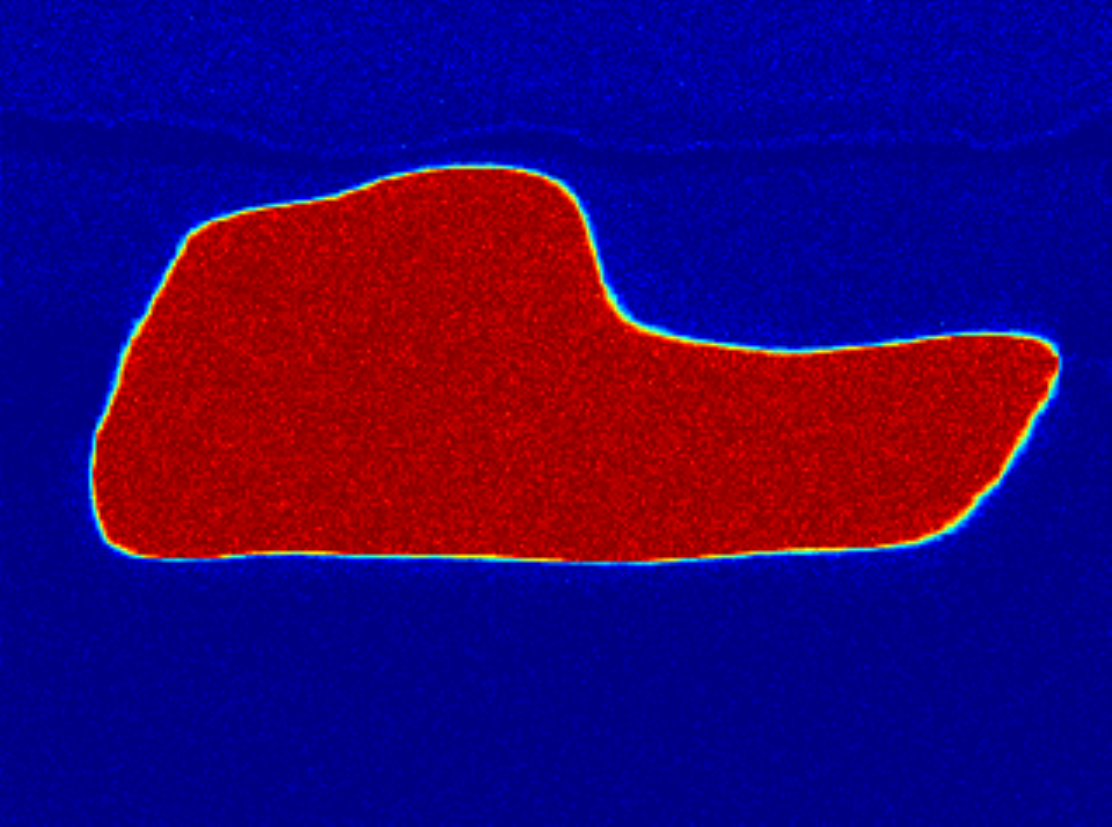}
\myfiguresixcol{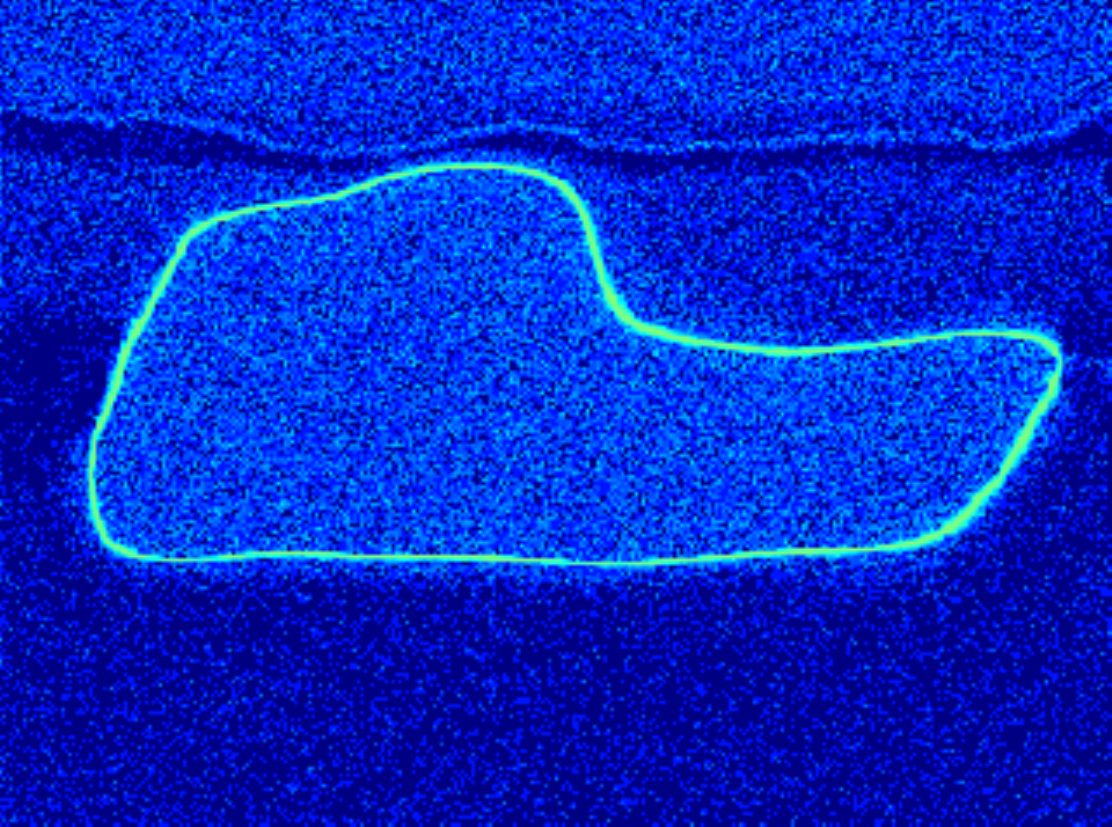}
\myfiguresixcol{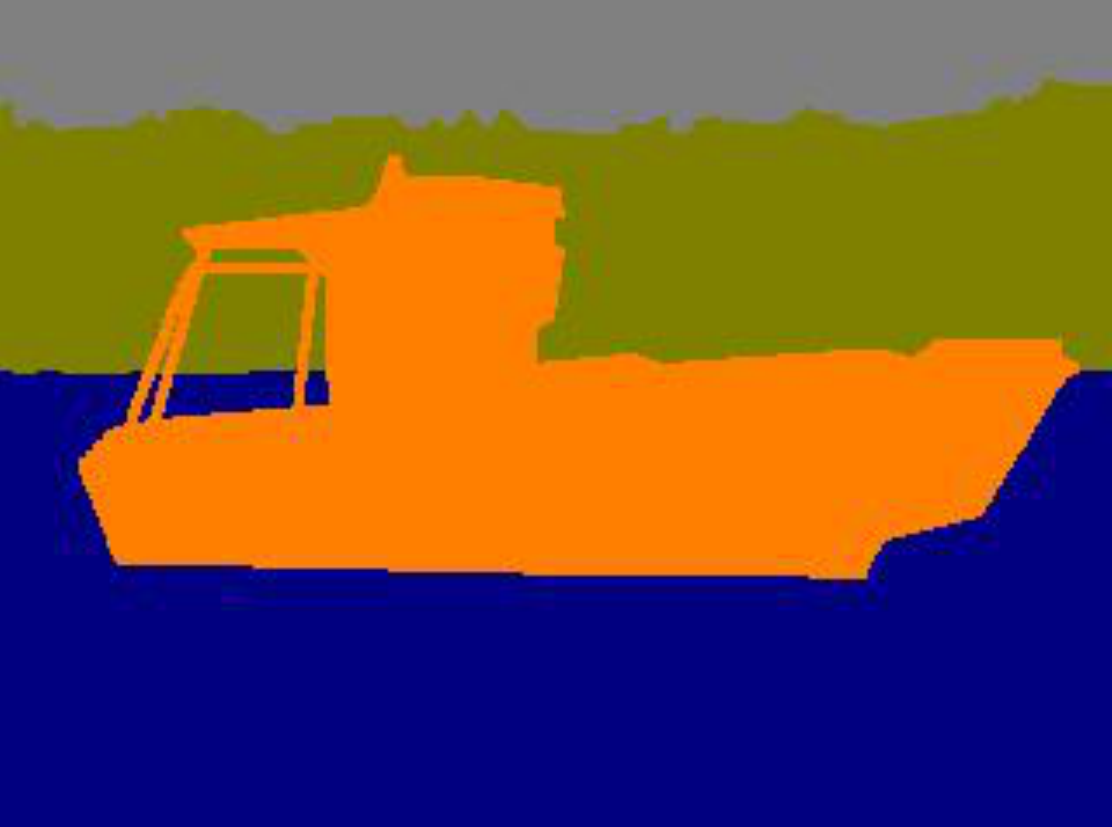}


\myfiguresixcol{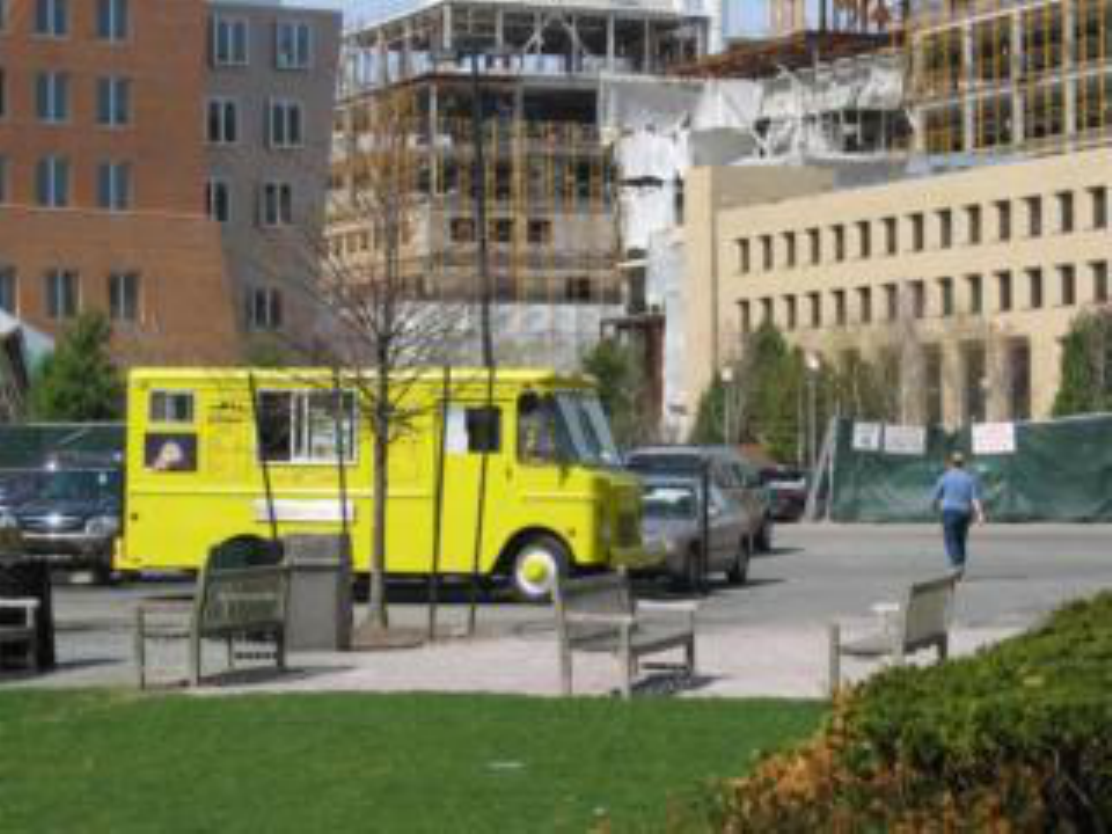}
\myfiguresixcol{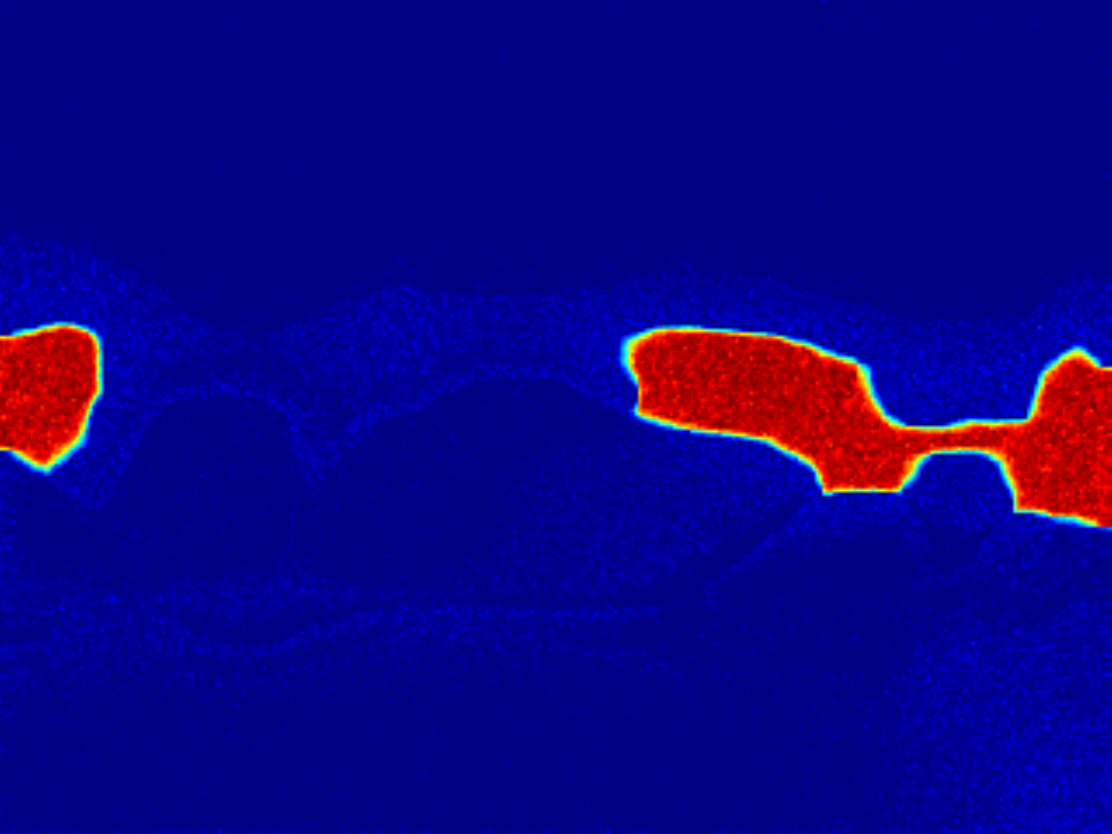}
\myfiguresixcol{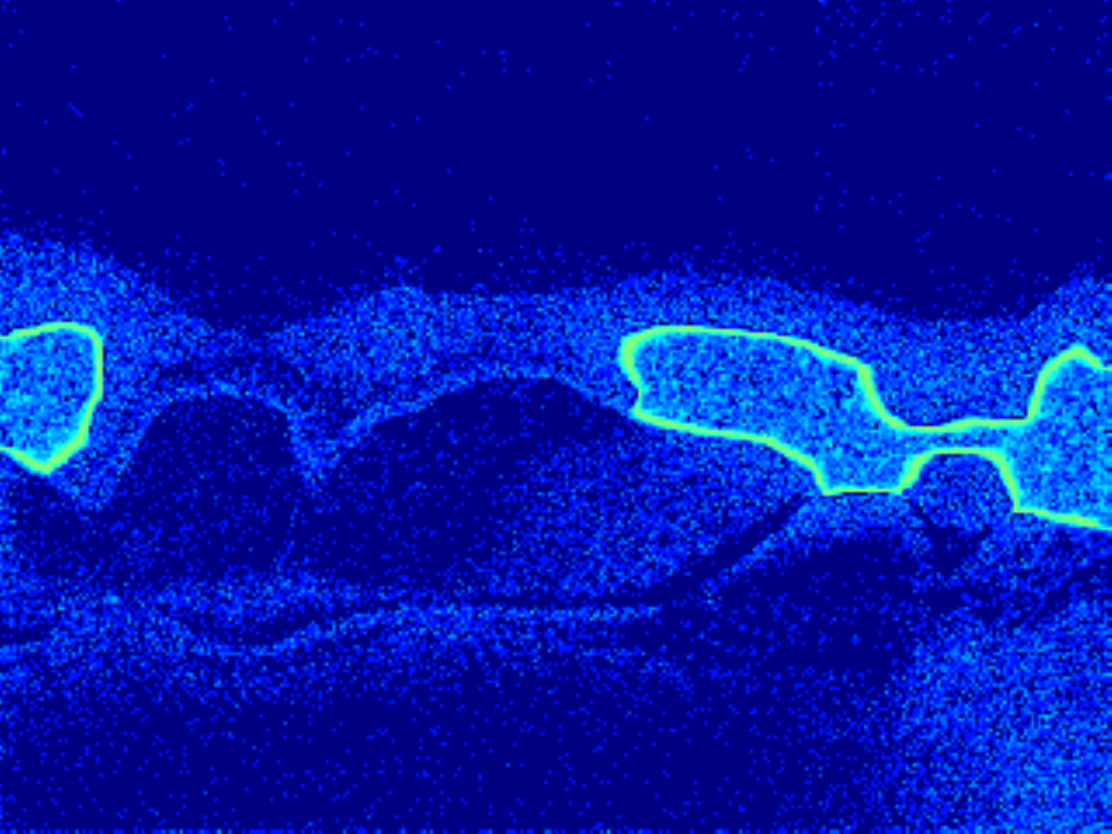}
\myfiguresixcol{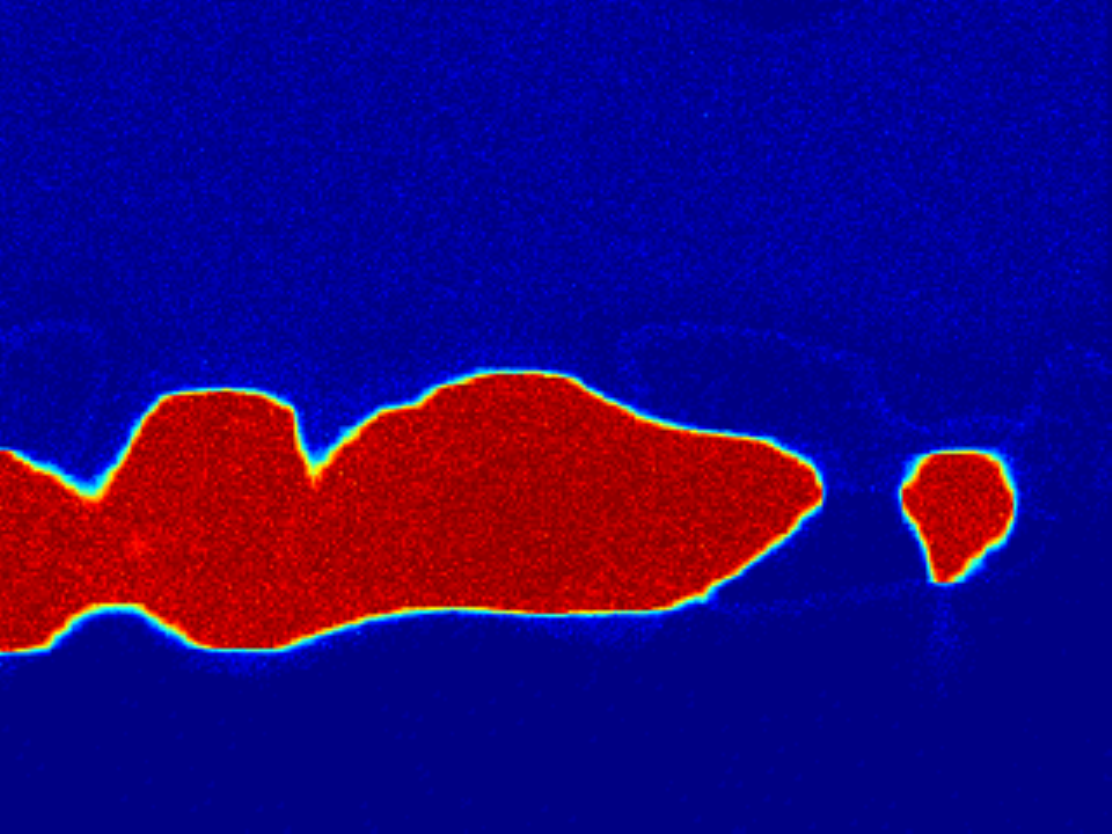}
\myfiguresixcol{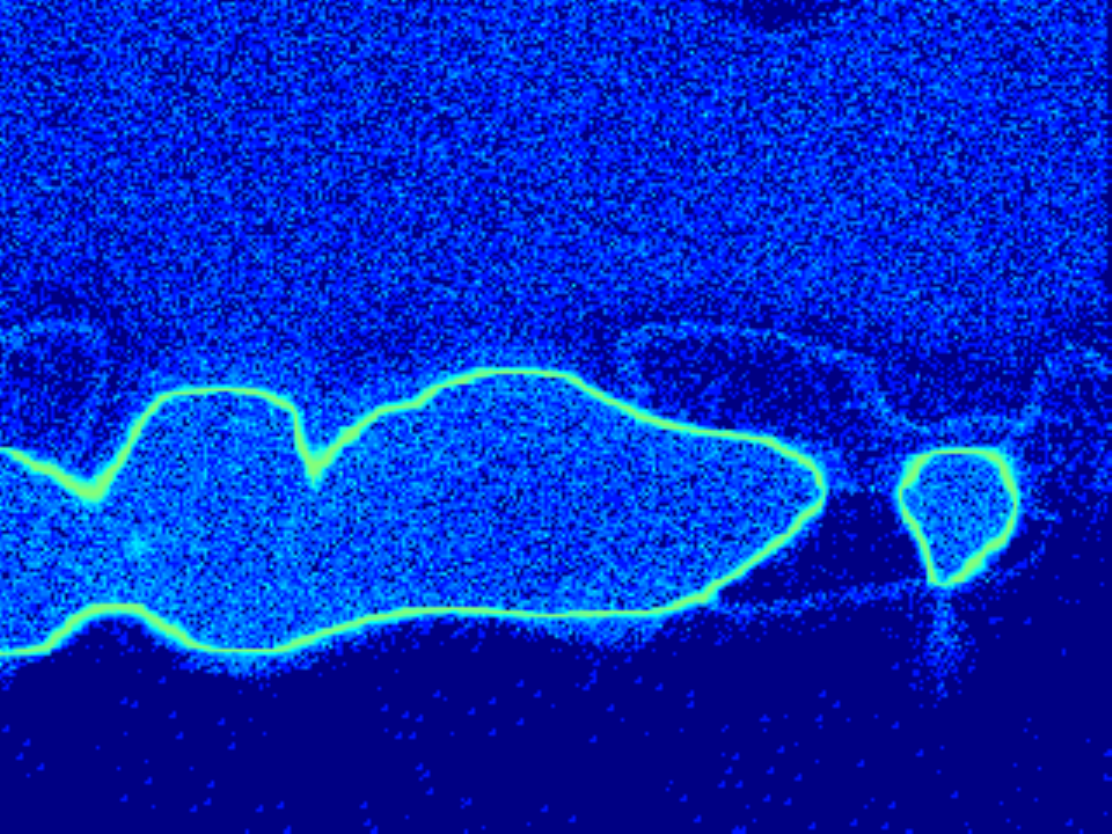}
\myfiguresixcol{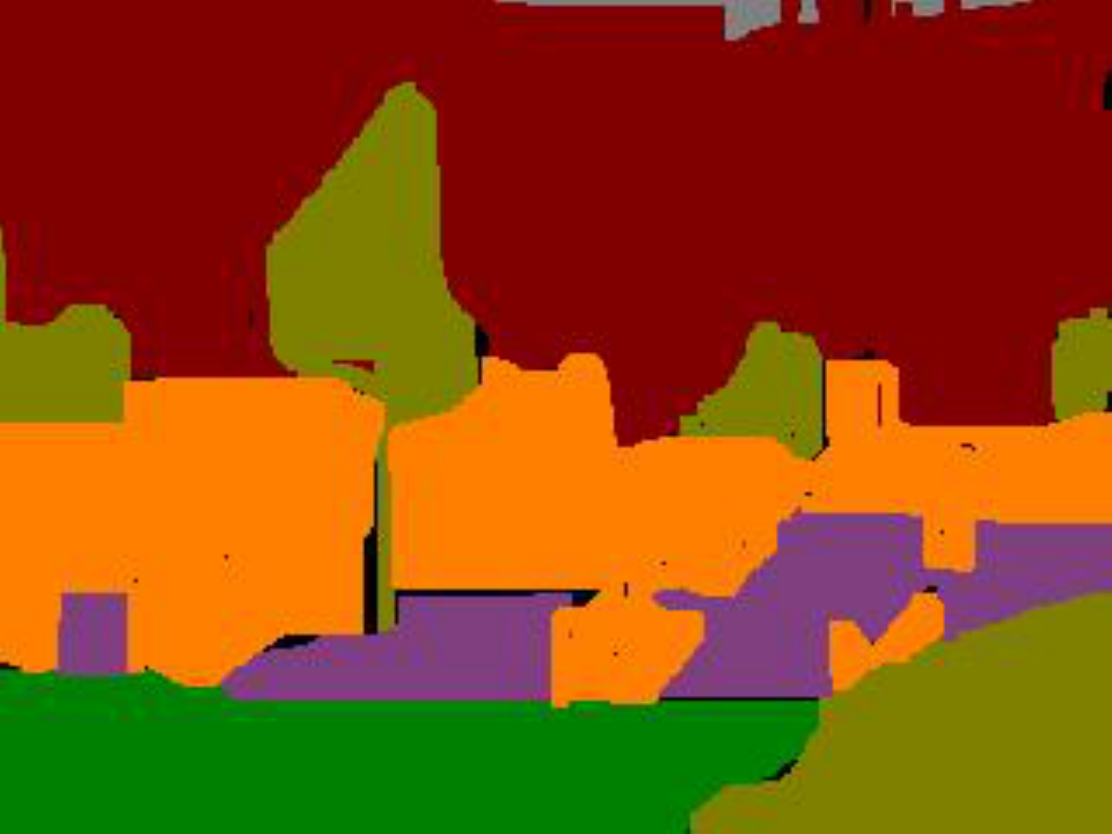}

\captionsetup{labelformat=default}
\setcounter{figure}{2}
    \caption{Figure that illustrates probabilities and variance of LocPMAP inference method  applied on deeply learned FCN unary features.  Note that while most other inference techniques produce a discrete solution to the problem, our LocPMAP method outputs discrete solution and also probabilities and variance for each pixel. In the the second and third columns, we show probabilities and variance for the ``Tree'' class predictions. The fourth and fifth columns depict the probabilities and the variance of the predictions for the ``Object'' class respectively. In the last column we show images corresponding to the ground truth.\vspace{-0.2cm}}
    \label{fig_prob_var}
\end{figure*}

Figure~\ref{fig_stanford} shows some qualitative results. Note that compared to the results achieved by~\citet{Gould+al:ICCV09}, our predictions are spatially smoother. Similarly, relative to the ICM predictions, Local Perturb-and-MAP yields crispier boundaries around the objects and more coherent segments. We also note that unlike most inference methods that can only predict the discrete label, our method outputs the prediction variance for every pixel in addition to the label (See Figure~\ref{fig_prob_var}). This probabilistic prediction may be useful in practical scenarios, such as for the analysis of failures or when confidence estimates are needed.





\section{Discussion}
We introduced a Local Perturb-and-MAP (LocPMAP) framework 
which yields a novel connection with pseudolikelihood (PL).
Our empirical analysis demonstrates that locPMAP forms a 
better inference procedure for models learned from PL optimization than  existing approximate inference methods. 
Future work includes extending our method to use larger blocks (corresponding to composite likelihood). Our new perspective on pseudolikelihood may  also be leveraged to solve challenging structure prediction problems in other domains.

{\small
\bibliography{gb_bibliography}

\begin{thebibliography}{37}
\providecommand{\natexlab}[1]{#1}
\providecommand{\url}[1]{\texttt{#1}}
\expandafter\ifx\csname urlstyle\endcsname\relax
  \providecommand{\doi}[1]{doi: #1}\else
  \providecommand{\doi}{doi: \begingroup \urlstyle{rm}\Url}\fi

\bibitem[Bertasius et~al.(2015)Bertasius, Shi, and Torresani]{gberta_2015_ICCV}
Bertasius, G., Shi, J., and Torresani, L.
\newblock High-for-low and low-for-high: Efficient boundary detection from deep
  object features and its applications to high-level vision.
\newblock In \emph{The IEEE International Conference on Computer Vision
  (ICCV)}, December 2015.

\bibitem[Besag(1975)]{Besag1975}
Besag, J.
\newblock Statistical analysis of non-lattice data.
\newblock \emph{Journal of the Royal Statistical Society. Series D (The
  Statistician)}, 24\penalty0 (3):\penalty0 pp. 179--195, 1975.
\newblock ISSN 00390526.

\bibitem[Chen et~al.(2015)Chen, Papandreou, Kokkinos, Murphy, and
  Yuille]{chen14semantic}
Chen, L.-C., Papandreou, G., Kokkinos, I., Murphy, K., and Yuille, A.~L.
\newblock Semantic image segmentation with deep convolutional nets and fully
  connected crfs.
\newblock In \emph{ICLR}, 2015.

\bibitem[Do \& Artieres(2010)Do and Artieres]{Do_AISTATS_2010}
Do, T.-M.-T. and Artieres, T.
\newblock Neural conditional random fields.
\newblock In \emph{Proceedings of the Thirteenth International Conference on
  Artificial Intelligence and Statistics}, volume~9, 5 2010.

\bibitem[Domke(2013)]{domke2013learning}
Domke, J.
\newblock Learning graphical model parameters with approximate marginal
  inference.
\newblock volume~35, pp.\  2454--2467. IEEE, 2013.

\bibitem[Donahue et~al.(2013)Donahue, Jia, Vinyals, Hoffman, Zhang, Tzeng, and
  Darrell]{DBLP:journals/corr/DonahueJVHZTD13}
Donahue, J., Jia, Y., Vinyals, O., Hoffman, J., Zhang, N., Tzeng, E., and
  Darrell, T.
\newblock Decaf: {A} deep convolutional activation feature for generic visual
  recognition.
\newblock \emph{CoRR}, abs/1310.1531, 2013.

\bibitem[Fei-Fei et~al.(2007)Fei-Fei, Fergus, and
  Perona]{Fei-Fei:2007:LGV:1235884.1235969}
Fei-Fei, L., Fergus, R., and Perona, P.
\newblock Learning generative visual models from few training examples: An
  incremental bayesian approach tested on 101 object categories.
\newblock \emph{Comput. Vis. Image Underst.}, 106\penalty0 (1):\penalty0
  59--70, April 2007.
\newblock ISSN 1077-3142.
\newblock \doi{10.1016/j.cviu.2005.09.012}.

\bibitem[Gane et~al.(2014)Gane, Hazan, and
  Jaakkola]{DBLP:conf/aistats/GaneHJ14}
Gane, A., Hazan, T., and Jaakkola, T.~S.
\newblock Learning with maximum a-posteriori perturbation models.
\newblock In \emph{Proceedings of the Seventeenth International Conference on
  Artificial Intelligence and Statistics, {AISTATS} 2014, Reykjavik, Iceland,
  April 22-25, 2014}, pp.\  247--256, 2014.

\bibitem[Gelfand(2014)]{gelfand2014bottom}
Gelfand, A.~E.
\newblock \emph{Bottom-Up Approaches to Approximate Inference and Learning in
  Discrete Graphical Models DISSERTATION}.
\newblock PhD thesis, UNIVERSITY OF CALIFORNIA, IRVINE, 2014.

\bibitem[Gould et~al.(2009)Gould, Fulton, and Koller]{Gould+al:ICCV09}
Gould, S., Fulton, R., and Koller, D.
\newblock Decomposing a scene into geometric and semantically consistent
  regions.
\newblock In \emph{Proceedings of the International Conference on Computer
  Vision (ICCV)}, 2009.

\bibitem[Hazan \& Jaakkola(2012)Hazan and Jaakkola]{DBLP:conf/icml/HazanJ12}
Hazan, T. and Jaakkola, T.~S.
\newblock On the partition function and random maximum a-posteriori
  perturbations.
\newblock In \emph{Proceedings of the 29th International Conference on Machine
  Learning, {ICML} 2012, Edinburgh, Scotland, UK, June 26 - July 1, 2012},
  2012.

\bibitem[Hazan et~al.(2013)Hazan, Maji, Keshet, and Jaakkola]{NIPS2013_5066}
Hazan, T., Maji, S., Keshet, J., and Jaakkola, T.
\newblock Learning efficient random maximum a-posteriori predictors with
  non-decomposable loss functions.
\newblock In Burges, C., Bottou, L., Welling, M., Ghahramani, Z., and
  Weinberger, K. (eds.), \emph{Advances in Neural Information Processing
  Systems 26}, pp.\  1887--1895. Curran Associates, Inc., 2013.

\bibitem[Jia et~al.(2014)Jia, Shelhamer, Donahue, Karayev, Long, Girshick,
  Guadarrama, and Darrell]{jia2014caffe}
Jia, Y., Shelhamer, E., Donahue, J., Karayev, S., Long, J., Girshick, R.,
  Guadarrama, S., and Darrell, T.
\newblock Caffe: Convolutional architecture for fast feature embedding.
\newblock \emph{arXiv preprint arXiv:1408.5093}, 2014.

\bibitem[Kirillov et~al.(2015)Kirillov, Schlesinger, Forkel, Zelenin, Zheng,
  Torr, and Rother]{DBLP:journals/corr/KirillovSFZ0TR15}
Kirillov, A., Schlesinger, D., Forkel, W., Zelenin, A., Zheng, S., Torr, P.
  H.~S., and Rother, C.
\newblock Efficient likelihood learning of a generic {CNN-CRF} model for
  semantic segmentation.
\newblock \emph{CoRR}, abs/1511.05067, 2015.

\bibitem[Krizhevsky et~al.(2012)Krizhevsky, Sutskever, and
  Hinton]{NIPS2012_4824}
Krizhevsky, A., Sutskever, I., and Hinton, G.~E.
\newblock Imagenet classification with deep convolutional neural networks.
\newblock In Pereira, F., Burges, C., Bottou, L., and Weinberger, K. (eds.),
  \emph{Advances in Neural Information Processing Systems 25}, pp.\
  1097--1105. Curran Associates, Inc., 2012.

\bibitem[Lafferty et~al.(2001)Lafferty, McCallum, and
  Pereira]{Lafferty:2001:CRF:645530.655813}
Lafferty, J.~D., McCallum, A., and Pereira, F. C.~N.
\newblock Conditional random fields: Probabilistic models for segmenting and
  labeling sequence data.
\newblock In \emph{Proceedings of the Eighteenth International Conference on
  Machine Learning}, ICML '01, pp.\  282--289, San Francisco, CA, USA, 2001.
  Morgan Kaufmann Publishers Inc.
\newblock ISBN 1-55860-778-1.

\bibitem[Lauritzen(1996)]{lauritzen1996graphical}
Lauritzen, S.~L.
\newblock \emph{Graphical models}.
\newblock Clarendon Press, 1996.

\bibitem[LeCun \& Cortes(2010)LeCun and
  Cortes]{lecun-mnisthandwrittendigit-2010}
LeCun, Y. and Cortes, C.
\newblock {MNIST} handwritten digit database.
\newblock 2010.

\bibitem[Lindsay(1988)]{lindsay1988composite}
Lindsay, B.~G.
\newblock Composite likelihood methods.
\newblock \emph{Contemporary mathematics}, 80\penalty0 (1):\penalty0 221--39,
  1988.

\bibitem[Long et~al.(2015)Long, Shelhamer, and Darrell]{long_shelhamer_fcn}
Long, J., Shelhamer, E., and Darrell, T.
\newblock Fully convolutional networks for semantic segmentation.
\newblock \emph{CVPR}, November 2015.

\bibitem[McFadden et~al.(1973)]{mcfadden1973conditional}
McFadden, D. et~al.
\newblock Conditional logit analysis of qualitative choice behavior.
\newblock 1973.

\bibitem[Meshi et~al.(2010)Meshi, Sontag, Jaakkola, and
  Globerson]{meshi2010learning}
Meshi, O., Sontag, D., Jaakkola, T., and Globerson, A.
\newblock Learning efficiently with approximate inference via dual losses.
\newblock International Machine Learning Society, 2010.

\bibitem[Papandreou \& Yuille(2011)Papandreou and
  Yuille]{conf/iccv/PapandreouY11}
Papandreou, G. and Yuille, A.~L.
\newblock Perturb-and-map random fields: Using discrete optimization to learn
  and sample from energy models.
\newblock In Metaxas, D.~N., Quan, L., Sanfeliu, A., and Gool, L. J.~V. (eds.),
  \emph{ICCV}, pp.\  193--200. IEEE, 2011.
\newblock ISBN 978-1-4577-1101-5.

\bibitem[Peng et~al.(2009)Peng, Bo, and Xu]{NIPS2009_3869}
Peng, J., Bo, L., and Xu, J.
\newblock Conditional neural fields.
\newblock In Bengio, Y., Schuurmans, D., Lafferty, J., Williams, C., and
  Culotta, A. (eds.), \emph{Advances in Neural Information Processing Systems
  22}, pp.\  1419--1427. Curran Associates, Inc., 2009.

\bibitem[Poon \& Domingos(2011)Poon and Domingos]{poon2011sum}
Poon, H. and Domingos, P.
\newblock Sum-product networks: A new deep architecture.
\newblock In \emph{Uncertainty in Artificial Intelligence}, pp.\  337?346,
  2011.

\bibitem[Srivastava et~al.()Srivastava, Salakhutdinov, and
  Hinton]{Srivastava_fastinference}
Srivastava, N., Salakhutdinov, R., and Hinton, G.
\newblock Fast inference and learning for modeling documents with a deep
  boltzmann machine.

\bibitem[Srivastava et~al.(2014)Srivastava, Hinton, Krizhevsky, Sutskever, and
  Salakhutdinov]{JMLR:v15:srivastava14a}
Srivastava, N., Hinton, G., Krizhevsky, A., Sutskever, I., and Salakhutdinov,
  R.
\newblock Dropout: A simple way to prevent neural networks from overfitting.
\newblock \emph{Journal of Machine Learning Research}, 15:\penalty0 1929--1958,
  2014.

\bibitem[Stoyanov \& Eisner(2012)Stoyanov and Eisner]{stoyanov2012minimum}
Stoyanov, V. and Eisner, J.
\newblock Minimum-risk training of approximate crf-based nlp systems.
\newblock In \emph{Proceedings of the 2012 Conference of the North American
  Chapter of the Association for Computational Linguistics: Human Language
  Technologies}, pp.\  120--130. Association for Computational Linguistics,
  2012.

\bibitem[Sutton \& McCallum(2009)Sutton and McCallum]{journals/ml/SuttonM09}
Sutton, C.~A. and McCallum, A.
\newblock Piecewise training for structured prediction.
\newblock \emph{Machine Learning}, 77\penalty0 (2-3):\penalty0 165--194, 2009.

\bibitem[Taigman et~al.(2014)Taigman, Yang, Ranzato, and Wolf]{deepface}
Taigman, Y., Yang, M., Ranzato, M., and Wolf, L.
\newblock Deepface: Closing the gap to human-level performance in face
  verification.
\newblock In \emph{Conference on Computer Vision and Pattern Recognition
  (CVPR)}, 2014.

\bibitem[Tarlow et~al.(2012)Tarlow, Adams, and Zemel]{journals/jmlr/TarlowAZ12}
Tarlow, D., Adams, R.~P., and Zemel, R.~S.
\newblock Randomized optimum models for structured prediction.
\newblock In Lawrence, N.~D. and Girolami, M. (eds.), \emph{AISTATS}, volume~22
  of \emph{JMLR Proceedings}, pp.\  1221--1229. JMLR.org, 2012.

\bibitem[Toshev \& Szegedy(2013)Toshev and
  Szegedy]{DBLP:journals/corr/ToshevS13}
Toshev, A. and Szegedy, C.
\newblock Deeppose: Human pose estimation via deep neural networks.
\newblock \emph{CoRR}, abs/1312.4659, 2013.

\bibitem[Wainwright(2006)]{journals/jmlr/Wainwright06}
Wainwright, M.~J.
\newblock Estimating the wrong graphical model: Benefits in the
  computation-limited setting.
\newblock \emph{Journal of Machine Learning Research}, 7:\penalty0 1829--1859,
  2006.

\bibitem[Xiang \& Neville()Xiang and Neville]{mismatch}
Xiang, R. and Neville, J.
\newblock On the mismatch between learning and inference for single network
  domains.

\bibitem[Xu \& Reid(2011)Xu and Reid]{xu2011robustness}
Xu, X. and Reid, N.
\newblock On the robustness of maximum composite likelihood estimate.
\newblock \emph{Journal of Statistical Planning and Inference}, 141\penalty0
  (9):\penalty0 3047--3054, 2011.

\bibitem[Yellott(1977)]{YELLOTT1977109}
Yellott, J.~I.
\newblock The relationship between luce's choice axiom, thurstone's theory of
  comparative judgment, and the double exponential distribution.
\newblock \emph{Journal of Mathematical Psychology}, 15\penalty0 (2):\penalty0
  109 -- 144, 1977.
\newblock ISSN 0022-2496.

\bibitem[Zheng et~al.(2015)Zheng, Jayasumana, Romera-Paredes, Vineet, Su, Du,
  Huang, and Torr]{crfasrnn_iccv2015}
Zheng, S., Jayasumana, S., Romera-Paredes, B., Vineet, V., Su, Z., Du, D.,
  Huang, C., and Torr, P.
\newblock Conditional random fields as recurrent neural networks.
\newblock In \emph{International Conference on Computer Vision (ICCV)}, 2015.

\end{thebibliography}
\bibliographystyle{icml2015mine}
}

\end{document}